\documentclass[11pt]{article}
\usepackage{bbm}
\usepackage{eqnarray,amsmath,amsfonts,amsthm,mathrsfs}
\usepackage{color}
\usepackage{bm}
\usepackage{amssymb}
\usepackage[dvips]{graphicx}
\usepackage{epsfig}
\usepackage{epsf}
\usepackage{float}
\usepackage{subfigure}
\usepackage{amsfonts,amsmath,amsthm,amssymb,graphicx,float,fancyhdr,multirow,hyperref}
\usepackage{color,booktabs,longtable,authblk}
\usepackage{mathrsfs,hhline}
\usepackage{algorithm}
\usepackage{algorithmic}

\usepackage{fancyhdr}
\usepackage[top=2.5cm, bottom=2.5cm, left=3cm, right=3cm]{geometry}
\usepackage{graphicx}
\newtheorem{theorem}{Theorem}
\newtheorem{assumption}{Assumption}
\newtheorem{corollary}{Corollary}
\newtheorem{definition}{Definition}

\newtheorem{lemma}{Lemma}[section]
\newtheorem{proposition}{Proposition}[section]
\newtheorem{remark}{Remark}[section]
\newtheorem{example}{Example}[section]

\allowdisplaybreaks[4]

\def\begeqn{\begin{equation}}
\def\endeqn{\end{equation}}
\def\begth{\begin{theorem}}
\def\endth{\end{theorem}}
\def\begprop{\begin{proposition}}
\def\endprop{\end{proposition}}
\def\begcor{\begin{corollary}}
\def\endcor{\end{corollary}}
\def\begdef{\begin{definition}}
\def\enddef{\end{definition}}
\def\beglemm{\begin{lemma}}
\def\endlemm{\end{lemma}}
\def\begexm{\begin{example}}
\def\endexm{\end{example}}
\def\begrem{\begin{remark}}
\def\endrem{\end{remark}}
\def\begassum{\begin{assumption}}
\def\endassum{\end{assumption}}

\def\si{\sigma}

\def\B{\mathcal{B}}

\def\O{\mathcal{O}}

\def\N{\mathbb{N}}
\def\R{\mathbb{R}}

\def\X{\mathcal{X}}
\def\Y{\mathcal{Y}}
\def\Z{\mathcal{Z}}
\def\E{\mathcal{E}}

\def\bE{\mathbb{E}}

\def\H{\mathcal{H}}

\def\B{\mathcal{B}}

\def\R{\mathbb{R}}

\def\RR{{\mathbb R}}

\title{Optimality of Robust Online Learning
$^\dag$\footnotetext{\dag~The corresponding author is Lei Shi.}}
\author{Zheng-Chu Guo$^1$, Andreas Christmann$^2$,  and Lei Shi$^3$\\
\small $^1$ School of Mathematical Sciences, Zhejiang University, Hangzhou 310058, P. R. China \\
Email: guozhengchu@zju.edu.cn\\
\small $^2$ Department of Mathematics, University of Bayreuth, Bayreuth 95447, Germany\\
Email: andreas.christmann@uni-bayreuth.de\\
\small $^3$  School of Mathematical Sciences and Shanghai Key Laboratory for Contemporary\\
\small Applied Mathematics, Fudan University, Shanghai 200433, P. R. China\\
Email: leishi@fudan.edu.cn}
\date{}
\begin{document}

\maketitle
\begin{abstract}
In this paper, we study an online learning algorithm with a robust loss function $\mathcal{L}_{\sigma}$ for regression over a reproducing kernel Hilbert space (RKHS). The loss function $\mathcal{L}_{\sigma}$ involving a scaling parameter $\sigma>0$ can cover a wide range of commonly used robust losses. The proposed algorithm is then a robust alternative for online least squares regression aiming to estimate the conditional mean function. For properly chosen $\sigma$ and step size, we show that the last iterate of this online algorithm can achieve optimal capacity independent convergence in the mean square distance. Moreover, if additional information on the underlying function space is known, we also establish optimal capacity dependent rates for strong convergence in RKHS. To the best of our knowledge,  both of the two results are new to the existing literature of online learning.
\end{abstract}

{\bf Keywords and Phrases:} Online learning, Robust regression, Convergence analysis, Reproducing kernel Hilbert space

{\bf Mathematics Subject Classification:} 68T05, 62J02, 68Q32, 62L20

\section{Introduction}

Online learning is one of the most popular approaches to handle large-scale datasets due to its low computational complexity and low storage requirements. Although great success achieved in batch learning, it becomes numerically intractable when the dataset is extremely large, as solving the optimization problem usually scales between quadratic and cubic complexity in the sample size. Instead of processing the entire training data in a batch, online learning can lead to prominent computational speed-up by tackling the data one by one. Recently, the scenario of online learning has attracted  tremendous interest and attention due to its successful applications in various fields \cite{Bottou18,Dieuleveut&Bach2016,sutskever2013importance,Ying2008,Zhang04}.

In this paper, we consider an online learning algorithm for robust kernel regression. As a non-parametric method developed during the last three decades, kernel regression in a reproducing kernel Hilbert space (RKHS) has a wide range of applications from machine learning to statistical inverse problems \cite{bauer2007regularization,Blanchard2016, Caponnetto2007,de2010adaptive,guo2017learning,lu2020,raskutti2014early}. Let $\rho$ be a Borel probability distribution on $\mathcal{X}\times\mathcal{Y}$, where $\mathcal{X}$ is an arbitrary, non-empty set equipped with $\sigma-$algebra and $\mathcal{Y}\subseteq \mathbb{R}$. The goal of non-parametric regression is to infer a functional relation between the explanatory variable $X$ that takes values in $\mathcal{X}$ and the response variable $Y\in \mathcal{Y}$, under the assumption that $\rho$ is the joint distribution of $(X,Y)$ but completely unknown. In most applications of regression analysis, the underlying functional relation of great importance is the conditional mean of $Y$ given $X=x$, namely the regression function. Denote by $\rho (y|x)$ the conditional distribution of $\rho$ for given $x\in \mathcal{X}$. The regression function is defined by
$$f_{\rho}(x)=\int_\X yd\rho(y|x), \quad \forall x \in \mathcal{X}.$$ Though $\rho$ is unknown, we have a sequence of samples $\{(x_t,y_t)\}_{t \in \N}$ independently distributed from $\rho$ instead. One typical way to estimate $f_{\rho}$ is empirical risk minimization in which an empirical error associated with the least squares loss is minimized at the given samples. However, from a robustness point of view, the least squares loss is not a good choice for regression, as it is not Lipschitz continuous, and thus the generated estimator can be dramatically affected by the smallest amount of outliers \cite{de2009robustness,debruyne2010robustness}. One of the main strategies to improve robustness is to replace the least squares loss by some robust alternatives, i.e., loss functions with bounded first derivatives. In this paper, we consider utilizing the loss function
\begin{equation}\label{lossfunction}
\mathcal{L}_\sigma(u)= W\left(\frac{u^2}{\sigma^2}\right)
\end{equation} to estimate $f_{\rho}$. Here $W:\R_+\mapsto \R$ is a windowing function and $\sigma>0$ is a scaling parameter. Moreover, the windowing function $W$ is required to satisfy the following two conditions:
\begin{equation}\label{condition1}
W_+'(0)>0,\  W'(s)>0 ~{\rm for}~s>0, \ C_W:=\sup_{s\in(0,\infty)}\{|W'(s)|\}<\infty,
\end{equation}
and there exist constants $p>0$ and $c_p>0$ such that
\begin{equation}\label{condition2}
|W'(s)-W'_+(0)|\le c_p |s|^p, \quad \forall s>0,
\end{equation} where $W'_+(0)$ denotes the right derivative of $W(x)$ at $x=0$.

By choosing different windowing functions, loss function of form (\ref{lossfunction}) can include a large variety of commonly used robust losses. We give some examples as follows, where $\mathbb{I}_{{A}}$ denote the indicator function of the set ${A}$.
\begin{itemize}
\item Fair loss \cite{Fair1974}: $\mathcal{L}_\sigma(u)=\frac{|u|}{\sigma}-\log\left(1+\frac{|u|}{\sigma}\right), W(s)=\sqrt{s}-\log(1+\sqrt{s}), p=\frac12,c_p=\frac12.$
\item Cauchy (aka. Lorentzian) loss \cite{black1996robust}: $\mathcal{L}_\sigma(u)=\log\left(1+\frac{u^2}{2\sigma^2}\right), W(s)=\log(1+\frac{s}{2}), p=1, c_p=\frac{1}{4}.$
\item Welsch loss \cite{Holland&Welsch1977}: $\mathcal{L}_\sigma(u)=1-\exp(-\frac{u^2}{2\sigma^2}), W(s)=1-\exp\left(-\frac{s}{2}\right), p=1, c_p=\frac14.$
\item Geman-McClure loss \cite{Ganan1985}: $\mathcal{L}_\sigma(u)=\frac{\sigma^2}{\sigma^2+t^2}, W(s)=\frac{1}{1+s}, p=1, c_p=2.$
\item Tukey's biweight loss \cite{HampelRonchettiRousseeuwStahel1986}: $\mathcal{L}_\sigma(u)=\frac{c^2}{6}\left(1-\left(1-\frac{u^2}{\sigma^2}\right)^3 {\mathbb{I}}_{\{|u|\le\sigma\}}\right), W(s)=\frac{c^2}{6}\big(1-(1-s)^3 {\mathbb{I}}_{\{s\le1\}}\big), p=1, c_p=c^2.$
\end{itemize}

Robust losses have been extensively studied in parametric regression, which leads to the development of robust statistics \cite{huber1981robust}. All the concrete examples of $\mathcal{L}_{\sigma}$ loss listed above were initially proposed in robust statistics to build robust estimator of linear regression. It should be pointed out that robust $\mathcal{L}_{\sigma}$ loss satisfying condition \eqref{condition1} and \eqref{condition2} could be non-convex, e.g.,  Cauchy and Welsh loss.  Empirical and theoretical studies show that non-convex $\ell_{\sigma}$ loss can lead to more robust estimators, compared with their convex counterparts, e.g., Huber's loss and Fair loss \cite{maronna2006robust,mizera2002breakdown}. This is mainly due to the redescending property, which can be illustrated by taking Welsh loss as an example. The roots of the second derivative of Welsh loss is $\pm \sigma$, which tells us at what value of $u$ the loss begins to redescend. When $|u|\leq \sigma$, Welsh loss is convex and behaves as the least squares loss; when $|u|> \sigma $ the loss function becomes concave and rapidly tends to be flat as $|u|\to \infty$. Therefore, with a suitable chosen scale parameter $\sigma$, Welsh loss can completely reject gross outliers while keeping a similar prediction accuracy as that of least squares loss, which makes it a more efficient robust loss. Recently, non-convex $\mathcal{L}_{\sigma}$ loss such as Welsh loss has drawn much attention in the signal processing community and shown its efficiency in non-parametric regression \cite{Feng2015,FengWu2021,HuangFegngWu2022,liu2007correntropy,lv2021,santamaria2006generalized}.

To construct robust non-parametric estimators of $f_{\rho}$, we propose an online learning algorithm to minimize the empirical error
$\frac{1}{T}\sum_{t=1}^T \mathcal{L}_{\sigma}(f(x_t)-y_t)$ for $T\in \mathbb{N}$ in an RKHS $\H_K$. The function space $\H_K$ is uniquely determined by a symmetric and positive semi-definite kernel $K:{\cal X} \times {\cal X} \to \mathbb{R}$ \cite{Aron}. Let $K_x: {\cal X} \to \mathbb{R}$ be the function defined by $K_x(s)=K(x,s)$ for $x,s\in {\cal X}$ and denote by  $\langle \cdot, \cdot \rangle_K$ the inner product of $\H_K$. Then $K_x \in \H_K$ and the reproducing property
\begin{equation}\label{reproducingproperty}
f(x)= \langle f, K_x \rangle_K
\end{equation} holds for all $x\in {\cal X}$ and $f\in {\cal H}_K$. The online algorithm considered in this work adopts a single-pass, fixed step-size stochastic gradient descent scenario in ${\cal H}_K$. Instead of computing the gradient of the empirical error with respect to the entire training set, the online learning algorithm only computes the gradient term of one sample randomly at each step. Given $z_t=(x_t,y_t)\in \mathcal{X}\times\mathcal{Y}$, the local error $\mathcal{L}_{\sigma}(f(x_t)-y_t)$ of $f\in \mathcal{H}_K$ at the sample $z_t$ can be regarded as a functional on $\mathcal{H}_K$. Due to the reproducing property of \eqref{reproducingproperty}, the gradient of $\mathcal{L}_{\sigma}(f(x_t)-y_t)$ at $f\in \mathcal{H}_K$ is $\mathcal{L}'_{\sigma}(f(x_t)-y_t)K_{x_t}$. Then the online learning algorithm is explicitly given by the following definition.
\begin{definition}\label{robustonlinealgorithm}
	Let $\{z_t=(x_t,y_t)\}_{t\in\N}$ be a sequence of random samples independently distributed according to $\rho.$ The online learning algorithm with the loss function  $\mathcal{L}_{\sigma}$ in \eqref{lossfunction} is defined by $f_1=0$, and
	\begin{align}\label{algorithm}
	f_{t+1}=f_t-\eta \mathcal{L}'_{\sigma}(f(x_t)-y_t)K_{x_t}=f_t-\eta W'\left(\xi_{t,\sigma}\right)(f_t(x_t)-y_t)K_{x_t}, \quad t\in \mathbb{N},
	\end{align}
	where $\eta>0$ is the step size and $\xi_{t,\sigma}=\frac{(y_t-f_t(x_t))^2}{\si^2}.$
\end{definition}

In kernel regression, the previous studies are mainly based on convex risk minimization to build robust estimators, where the empirical error with a convex robust loss adding a regularization term is minimized in some infinite-dimensional RKHS. The typical examples include support vector machines with Huber's loss, logistic loss, absolute value loss and its asymmetric variant known as the pinball loss \cite{ChristmannVanMessemSteinwart2009,Christmann2007}. In general, these surrogate robust losses for least squares loss can not be used to estimate the regression function unless the conditional distributions of $Y|X=x$ are known to be symmetric \cite{Steinwart2007}. However, as a basic estimator in data analysis, regression function is used in many situations for forecasting, modeling and analysis of trends, which is of most interest to us. On the other hand, computation of robust estimates is much more computationally intensive than least squares estimations, especially in large-scale data analysis. The online learning algorithm in Definition \ref{robustonlinealgorithm} can provide a robust estimator of regression function for large-scale data analysis. The computational cost of this algorithm is $\mathcal{O}(T^2)$ when the sample size is $T$, that is, the algorithm terminates after $T$ iterations with final output $f_{T+1}$. At each iteration, the main computational cost is due to the evaluation of $f_t(x_t)$ which needs to calculate $K(x_i, x_t)$ for $i$ from 1 to $t$. If one can compute and store all $\{K(x_i,x_j)\}_{i,j=1}^T$ in advance, the computational cost can be reduced to linear $\mathcal{O}(T)$ at the requirement of large memory and fast memory access.

In this paper, we aim to evaluate the optimality of robust online learning algorithm in Definition \ref{robustonlinealgorithm}. Under the framework of non-parametric regression, the stochastic optimality of online learning and its variants have been extensively studied in a vast literature, see e.g., \cite{Dieuleveut&Bach2016,Lin2016,yao2010complexity,Ying2008,ying2017unregularized}. However, the developed theoretical analysis only focuses on least squares loss or other convex losses. When $\mathcal{L}_{\sigma}$ is a non-convex loss function, the iteration sequences produced by online learning algorithm are often trapped at stationary points of the objective function, which brings essential difficulties to the mathematical analysis. As far as we know, no optimality analysis so far have been given to support the efficiency of online learning with robust $\mathcal{L}_{\sigma}$ loss of form (\ref{lossfunction}). In this work, we show that with an appropriately chosen scale parameter $\sigma$, the iteration sequences of constant step-size robust online learning can approximate the regression function $f_{\rho}$. The approximation accuracy is measured by the rates of convergence in the standard mean square distance and the strong RKHS norm. We shall establish a novel convergence analysis by fully exploiting the properties of the loss function $\mathcal{L}_{\sigma}$ and the structure of the underlying RKHS. Our analysis of the convergence is tight as the derived upper bounds on the performance of robust online learning almost match the minimax lower bounds in batch least squares regression learning. Especially, we obtain the capacity dependent optimal rates of strong convergence in the sense of RKHS norm. This kind of convergence is very important but seldom considered in previous studies of online learning.

The rest of this paper will be organized as follows. We present main results in Section \ref{section: main results}. Discussions and comparisons with related work are given in Section \ref{section: related work}.  Section \ref{section: error decomposition} establishes an error decomposition of the algorithm (\ref{algorithm}) as well as some basic estimates which is useful for our convergence analysis. The proofs of main results are given in Section \ref{section: proof of main results}.

\section{Main Results}\label{section: main results}

Before presenting main results, we first introduce some notations and assumptions. In our setting, the input space $\cal X$ is a general measurable space with the $\sigma-$algebra $\cal{A}$ and $K:\cal{X} \times \cal{X} \to \mathbb{R}$ is a symmetric and positive semi-definite kernel. We suppose that the kernel $K(\cdot,\cdot)$ is measurable on ${\cal X} \times {\cal X}$ for the product $\sigma-$algebra $\cal{A} \otimes \cal{A}$, and $\kappa:=\sup_{x\in \cal X}\sqrt{K(x,x)}<\infty$. Therefore, the underlying RKHS ${\cal H}_K$ consists of bounded measurable real-valued functions on $\cal{X}$ and the function $x \to K(x,x)$ is measurable on $({\cal X},{\cal A})$ (see, for instance, \cite{Steinwart2008}). Let $\rho_{\cal X}$ be the marginal distribution of $\rho$ on ${\cal X}$ and $L^2_{\rho_{\cal X}}$ be the Hilbert space of square-integrable functions with respect to $\rho_{\cal X}$. Denote by $\|\cdot\|_{\rho}$ the norm in the space $L^2_{\rho_{\cal X}}$  induced by the inner product $\langle f, g\rangle_{\rho}=\int_{\cal X} f(x) g(x) d\rho_{\cal X}(x)$. Since $\int_{{\cal X}} K(x,x) d\rho_{\cal X}(x)\leq \kappa^2$, i.e., $K$ is integrable on the diagonal, the RKHS ${\cal H}_K$ is compactly embedded into $L^2_{\rho_{\cal X}}$. Then the integral operator $L_K: L^2_{\rho_{\cal X}} \to L^2_{\rho_{\cal X}}$, given by, for $f\in L^2_{\rho_{\cal X}}$ and $u\in \cal X$,
\begin{equation*}
L_K(f)(u)=\int_{\cal X} f(x)K(x,u)d\rho_{\cal X}(x),
\end{equation*} is a compact, self-adjoint, and positive operator on $L^2_{\rho_{\cal X}}$. Due to the Spectral theorem, there exists in $L^2_{\rho_{\cal X}}$ an orthonormal basis $\{\phi_k\}_{k \geq 1}$ consisting of eigenfunctions of $L_K$, and the corresponding eigenvalues $\{\lambda_k\}_{k \geq 1}$ (repeated according to their algebraic multiplicity) are nonnegative. Then we can define the $r-$th power of $L_K$ on $L^2_{\rho_{\X}}$ by $L^r_K(\sum_{k \geq 1} c_k \phi_k)=\sum_{k \geq 1} c_k \sigma^r_k \phi_k$ with $r>0$ and $\{c_k\}_{k\geq 1} \in \ell^2(\mathbb{R})$ (i.e., $\sum_{k\geq 1}c^2_k<\infty$). In particular, $L^{1/2}_K$ is an isomorphism from $\overline{{\cal H}_K}$, the closure of ${\cal H}_K$ in $L^2_{\rho_{\cal X}}$, to ${\cal H}_K$, i.e., for each $f\in \overline{{\cal H}_K}$, $L^{1/2}_K f \in {\cal H}_K$ and
\begin{equation}\label{normrelation2}
\|f\|_{\rho}=\|L^{1/2}_K f\|_K.
\end{equation} Moreover, for all $f\in L^2_{\rho_{\cal X}}$, we have $L_K f \in {\cal H}_K$. In view of the above discussions, the operator $L_K$ can also be interpreted as an operator on ${\cal H}_K$.
For simplicity, whether it is viewed as an operator on $L^2_{\rho_{\cal X}}$ or ${\cal H}_K$, we will keep the same notation. In both cases, $L_K$ is a nuclear operator under our assumptions.

Our error analysis is based on the following regularity condition on the target function $f_\rho$, which is classical in the literature of kernel regression \cite{Caponnetto2007}.
\begin{assumption}\label{assumption1}
\begin{equation}\label{regularity condition}
f_\rho= L_K^r g_\rho \quad \mbox{with $r>0$ and $g_\rho\in L_{\rho_{\X}}^2$}.
\end{equation}
\end{assumption} This assumption implies that $f_{\rho}$ belongs to the range space of $L^r_K$ expressed as
$$L^r_{K}(L^2_{\rho_{\X}})=\left\{f \in L^2_{\rho_{\X}}: \sum_{k \geq 1} \frac{\langle f, \phi_k\rangle^2_{\rho}}{\lambda_k^{2r}}<\infty \right\}.$$ Then $L^{r_1}_{K}(L^2_{\rho_\X}) \subseteq L^{r_2}_{K}( L^2_{\rho_\X})$ whenever $r_1 \geq r_2$. The regularity of $f_{\rho}$ is measured by the decay rate of its expansion coefficients in terms of $\{\phi_k\}_{k \geq 1}$. Condition (\ref{regularity condition}) means that $\langle f, \phi_k\rangle^2_{\rho_X}$ decays faster than the $2r-$th power of the eigenvalues of $L_K$. Apparently, larger parameters $r$ will result in faster decay rates, and thus indicate higher regularities of $f_{\rho}$. This assumption is standard in the literature of learning theory, and can be further interpreted by the theory of interpolation spaces \cite{Smale2003}.

Throughout the paper, we assume that the output $y$ is uniformly bounded, i.e., for some constant $M>0$, $|y|\le M$ almost surely. Recall that the sample $\{z_t=(x_t,y_t)\}_{t\in\N}$ is drawn independently from the probability distribution $\rho$. For $k\in\N,$ let $\bE_{z_1,\cdots,z_k}$ denote taking expectation with respect to $z_1,\cdots,z_k$, which is written as $\bE_{Z^k}$ for short. Our first main result is related to the convergence in the mean square distance, which establishes upper bounds of $L^2_{\rho_\X}$ norm in expectation.

\begin{theorem}\label{thm: convergence rate in L2}
Let $\{f_t\}_{t=1}^{T+1}$ be defined by algorithm (\ref{algorithm}) with a windowing function $W$. Assume that $W$ satisfies \eqref{condition1} and \eqref{condition2} with some $p>0$ and the regularity condition (\ref{regularity condition}) holds with $r>0$. Choose the step size $\eta=\frac{1}{\eta_{0}}T^{\frac{-2r}{2r+1}}$ with
\begin{align}\label{condition for eta0}
\eta_0\ge \max\left\{{C_W}\kappa^2,{ \left(\frac{1}{e}+2\kappa^2 W_+'(0)\right)^2}\right\}.
\end{align}
Then
\begin{align*}
\bE_{Z^T}\left[\|f_{T+1}-f_\rho\|_\rho^2\right]\le C \max\left\{T^{-\frac{2r}{2r+1}}\log T, T^{\frac{2p+2}{2r+1}}\sigma^{-4p}\right\},
\end{align*}
where $C_W:=\sup_{s\in(0,\infty)}\{|W'(s)|\}$ and the constant $C$ is independent of $T$ and will be given explicitly in the proof.
\end{theorem}
Especially, let $\sigma\ge T^\frac{r+p+1}{2p(2r+1)},$ the rate of convergence is at least $\mathcal{O}\left(T^{-\frac{2r}{2r+1}}\log T\right)$, which is almost  minimax optimal (up to a logarithmic term) due to the discussion in \cite{Ying2008}. The bound presented in Theorem \ref{thm: convergence rate in L2} is called capacity independent as we can establish it without requiring further assumptions except for the boundedness of kernel functions.

Moreover, if we know some additional information about the kernel $K$, or the capacity of the hypothesis space $\H_K$, we can establish optimal capacity dependent convergence rates in $\H_K$ in the minimax sense. As pointed out in \cite{Smale2007learning}, the convergence in ${\cal H}_K$ implies the convergence in $C^n (\cal{X})$ if $K\in C^{2n}({\cal X}\times {\cal X})$. Here $n\in \mathbb{N}$ and $C^n(\cal{X})$ is the space of all functions on ${\cal X}\subset\RR^d$ whose partial derivatives up to order $n$ are continuous with $\|f\|_{C^n(\cal{X})}=\sum_{|s|\le n}\|D^s f\|_\infty.$ So the convergence in ${\cal H}_K$ is much stronger, which ensures that the estimators can not only approximate the regression function itself, but also approximate its derivatives. In this paper, we use the following condition to measure the capacity of the hypothesis space $\H_K$.

\begin{assumption}\label{assumption2}
For $0<\beta<1,$ we assume
\begin{equation}\label{capacity condition}
{\rm Tr}(L_K^\beta)<\infty.
\end{equation}
Here ${\rm Tr}(A)$ denotes the trace of the operator $A.$
\end{assumption}
The definition of $L_K^\beta$ gives that ${\rm Tr}(L_K^\beta)=\sum_{k\ge 1}\lambda_k^\beta$, here $\{\lambda_k\}_{k\in\N}$ are the eigenvalues of $L_K$. Since $L_K$ is a trace class operator satisfying ${\rm Tr}(L_K)=\sum_{k\geq 1} \lambda_k = \int_\X K(x,x)d\rho_\X (x) \leq \kappa^2$,  the capacity assumption (\ref{capacity condition}) holds trivially with $\beta=1$. Thus the case of $\beta=1$ corresponds to the capacity independent case as we don't require any complexity
measures of the underlying function space. Capacity assumption (\ref{capacity condition}) incorporates the information of marginal distribution $\rho_\X$, which is a tighter measurement for the complexity of the RKHS than more classical covering, or entropy number assumptions \cite{CZ2007,Steinwart2008}. This assumption is essentially an eigenvalue decaying condition imposed on the operator $L_K$. In fact, if ${\rm Tr}(L_K^\beta)<\infty,$ since the eigenvalues $\{\lambda_k\}_{k\in \mathbb{N}}$ are sorted in a decreasing order, then for any $k\ge 1,$ we have
$$k\lambda_k^\beta\le \sum_{j=1}^k\lambda_j^\beta\le \sum_{j\ge 1}\lambda_j^\beta={\rm Tr}(L_K^\beta)<\infty.$$ It follows that $\lambda_k\le k^{-1/\beta}({\rm Tr}(L_K^\beta))^{1/\beta}, \forall k \geq 1$. A small value of $\beta$ implies a fast polynomially decaying rate at least achieved by the eigenvalues $\{\lambda_k\}_{k\in \mathbb{N}}$. One can refer to Theorem 5 in our recent work \cite{guo2022capacity}, which provides a characterization of the relationship between the capacity assumption \eqref{capacity condition} in this paper and the decaying rate of integral operator eigenvalues. If the eigenvalues decay exponentially, the index $\beta$ can be arbitrarily close to zero. The polynomially decaying of the eigenvalue is typical for the Sobolev smooth kernels on domains in Euclidean spaces, and the parameter $\beta$ depends on the smoothness of the kernel $K.$ While the exponentially decaying integral operator is typical for the analytic kernels on domains in Euclidean spaces.  Recently, under the same capacity assumption, the work \cite{Pillaud2018} studies the mean square convergence of averaging online estimator with least square loss and multiple passes.

Under the regularity condition (\ref{regularity condition}) on the target function $f_\rho$ and the capacity assumption (\ref{capacity condition}) on the hypothesis space $\H_K,$ we obtain the following sharp capacity dependent results for strong convergence in $\H_K$.
\begin{theorem}\label{thm: convergence rate in HK}
Let $\{f_t\}_{t=1}^{T+1}$ be defined by algorithm (\ref{algorithm}) with a windowing function $W$. Assume that $W$ satisfies \eqref{condition1} and \eqref{condition2} with some $p>0$, the regularity condition (\ref{regularity condition}) holds with $r> \frac12$ and the capacity assumption (\ref{capacity condition}) holds with $0<\beta<1$. Choose step size $\eta=\frac{1}{\eta_{0}}T^{\frac{1-2r-\beta}{2r+\beta}}$ with
\begin{equation*}
\eta_0\ge \max\left\{{C_W}\kappa^2,{ \left(\frac{1}{e}+2\kappa^2 W_+'(0)\right)^2}\right\}.
\end{equation*}
Then
\begin{align*}
\bE_{Z^T}\left[\|f_{T+1}-f_\rho\|_K^2\right]\le \tilde{C} \max\left\{T^{-\frac{2r-1}{2r+\beta}},T^{\frac{2p+3}{2r+\beta}}\sigma^{-4p}\right\},
\end{align*}
where $C_W:=\sup_{s\in(0,\infty)}\{|W'(s)|\}$ and the constant $\tilde{C}$ is independent of $T$ and will be given explicitly in the proof.
\end{theorem}
If we choose $\sigma\ge T^\frac{p+r+1}{2p(2r+\beta)}$, the convergence rate is of the form $\O(T^{-\frac{2r-1}{2r+\beta}})$ which is tight as it matches the minimax lower bounds of non-parametric regression in batch learning \cite{Blanchard2016}. As far as we know, this is the first capacity dependent optimal rates obtained for strong convergence of online learning. We will prove all these results in Section \ref{section: proof of main results}.

\section{Discussion on Related Work}\label{section: related work}

There has been a recent surge of research that applies and explores specific $\mathcal{L}_{\sigma}$ losses (e.g., Cauchy and Welsh loss) in the context of non-parametric regression \cite{bessa2009entropy,he2011maximum,liu2007correntropy,santamaria2006generalized,sun2010secrets}. Convergence properties of these methods are also the subject of intense study but only limited to the classical batch learning setting, in which we collect all the samples ${\bf z}:=\{(x_i,y_i)\}_{i=1}^T$ initially and perform estimation only once. In \cite{Feng2015}, an empirical risk minimizer $f_{\bf z}$ with Welsh loss is considered, aiming to estimate the regression function $f_{\rho}$ over a compact hypothesis space ${\cal H}$. The theoretical analysis in \cite{Feng2015} shows that, if $f_{\rho} \in {\cal H}$ and the logarithm of covering number of ${\cal H}$ satisfies a polynomial increasing condition with a power index $0<\alpha\leq 2$, by choosing $\sigma=T^{\frac{1}{2+\alpha}}$,  probabilistic bounds established for $\|f_{\bf z}-f_\rho\|_\rho^2$ can converge to zero at a rate of $T^{-\frac{2}{2+\alpha}}$.
To make a comparison, given a Mercer kernel $K$ on a compact metric space ${\cal X}$, i.e., $K$ is continuous, symmetric and positive semi-definite on ${\cal X} \times {\cal X}$, let ${\cal H}$ be a bounded ball in an RKHS ${\cal H}_K$. Then ${\cal H}$ is a compact set consisting of continuous functions on ${\cal X}$, and the covering number condition of ${\cal H}$ is satisfied with $\alpha=2$, which leads to the convergence rate of the form $\mathcal{O}(T^{-\frac12})$. The covering number condition with $\alpha=2$ corresponds to capacity independent case in our analysis and $f_{\rho}\in {\cal H}$ implies that the regularity condition \eqref{assumption1} is satisfied with some $r\geq \frac{1}{2}$. Then Theorem \ref{thm: convergence rate in L2} in our paper asserts that the estimator of online algorithm \ref{algorithm} will approximate $f_{\rho}$ in $L^2_{\rho_{\cal X}}-$norm at a convergence rate $\mathcal{O}\left(T^{-\frac{2r}{2r+1}}\log T\right)$, which is faster than $\mathcal{O}(T^{-\frac12})$ provided that $r>\frac{1}{2}$. It should be pointed out that the convergence analysis in \cite{Feng2015} needs $f_{\bf z}$ to be a global minimizer of the empirical risk, but existing approaches applied to solve the corresponding minimization problem can not guarantee the global optimality due to the non-convexity of Welsh loss. To fill this gap, a gradient descent algorithm with robust loss function $\mathcal{L}_{\sigma}$ is proposed in \cite{GuoHuShi}, which is defined as $g_1=0,$ and
\begin{align}\label{gradident descent algorithm}
	g_{t+1}=g_t-\frac{\eta_t}{T}\sum_{i=1}^T W'\left(\xi_{t,\sigma}\right)(g_t(x_i)-y_i)K_{x_i}, \quad t\in \mathbb{N}.
\end{align}
where $\xi_{t,\sigma}=\frac{(y_i-g_t(x_i))^2}{\si^2}$ and $\eta_t>0$ is the step size. It is shown in \cite{GuoHuShi} that, with an appropriately chosen scale parameter $\sigma$ and early stopping rule $\ell\in \mathbb{N}$, the output of \eqref{gradident descent algorithm} after $\ell$ iterates can approximate $f_{\rho}$ under regularity condition \eqref{regularity condition} and an eigenvalue decaying condition of $L_K$. As mentioned in Section \ref{section: main results}, Assumption \ref{assumption2} with $0<\beta<1$ adopted in this paper is essentially an eigenvalue decaying condition, which implies that the eigenvalues of $L_K$ decay polynomially as $\lambda_k \leq c_{\beta} k^{-1/\beta},$ for all $k \geq 1$ and some $c_{\beta}>0$. Then under Assumption \ref{assumption1} and Assumption \ref{assumption2} in this paper, capacity dependent rate-optimal convergence analysis in both $L_{\rho_{\cal X}}^2-$norm and ${\cal H}_K-$norm is established in \cite{GuoHuShi}. We take strong convergence in ${\cal H}_K-$norm for instance to illustrate the above results.  Let $\eta_t=\eta_1 t^{-\theta}$ with $0\le \theta<1$ and some positive constant $\eta_1$. Choose the early stopping rule as $\ell=\lceil T^{\frac{1}{(1+\beta)(1-\theta)}}+1\rceil$, and $\sigma \geq T^{\frac{r+(p+1)\beta}{2p(2r+\beta)}}$ where $\lceil x \rceil$ denotes
the smallest integer not less than $x\in \mathbb{R}$. Then $\|g_{\ell+1}-f_{\rho}\|^2_K$ will converge to zero at a rate of $T^{-\frac{2r-1}{2r+\beta}}$ which is exactly the same as that obtained in Theorem \ref{thm: convergence rate in HK}. Therefore, gradient descent algorithm \eqref{gradident descent algorithm} and our online algorithm \eqref{algorithm} are both provably statistical optimal. Furthermore, both of these two algorithms are plug-and-play: only a loss and its gradient are necessary to integrate an optimization process to approximate  $f_{\rho}$. However, since the gradient decent algorithm (\ref{gradident descent algorithm}) is designed for batch learning, it still suffers from scalability issues. To see this, if algorithm stops after $\ell$ iterations and all $\{K(x_i,x_j)\}_{i,j=1}^T$ are computed and stored in advance, the aggregate time complexity is $\mathcal{O}(\ell T^2)$ scaling between $\mathcal{O}(T^2)$ and $\mathcal{O}(T^3)$, as calculation of gradient in each iteration involves all the training sample. Under the same situation,  online learning algorithm (\ref{algorithm}) requires only one training sample to update, which enjoys linear $\mathcal{O}(T)$ complexity but comparable theoretical performance. Last but not the least, convergence analysis in this paper is developed under more general setting: we only need the input space ${\cal X}$ to be an arbitrary measurable space and the positive semi-definite kernel function $K$ to be bounded. Capacity dependent analysis of algorithm \eqref{gradident descent algorithm} established in \cite{GuoHuShi} essentially relies on the classical Mercer theorem to provides a link between the spectral properties of $L_K$ and the capacity as well as approximation ability of ${\cal H}_K$. Thereby, estimates in \cite{GuoHuShi} requires ${\cal X}$ to be a compact metric space and $K$ to be a Mercer kernel. These settings are too restrictive if one has to perform regression analysis on more general ${\cal X}$ such as the set of graphs and strings (see e.g.,\cite{Scholkopf2002} and the references therein). While our capacity dependent analysis can provide a rigorous theoretical demonstration to support the efficiency of online kernel regression \eqref{algorithm} in a wider range of applications.

There is an extensive research on the minimax optimality of online learning in non-parametric regression, however almost all of them only focus on least squares loss. Next we review these work of particular relevance and make some comparisons. Online learning with least squares loss is defined as, $g_1=0,$ and
\begin{equation}\label{online learning with least square}
g_{t+1}=g_t-\eta_t(g_t(x_t)-y_t)K_{x_t},\qquad t\in \mathbb{N}.
\end{equation}
The first category of convergence analysis for algorithm (\ref{online learning with least square}) only gives capacity independent error bounds. Asides from the regularity conditions for $f_\rho$ stated in Assumption \ref{assumption1}, there is no assumptions imposed on the capacity of the underlying RKHS $\H_K$ to derive those error bounds. In \cite{Ying2008}, algorithm (\ref{online learning with least square}) was thoroughly investigated via a capacity independent approach, the performance of the last iterate
with polynomially decaying step sizes and constant step sizes was studied. It shows in \cite{Ying2008} that if Assumption \ref{assumption1} is satisfied with $r>0$, and take a constant step size with $\eta_t=\eta:=[r/64(1+\kappa)^4(2r+1)]T^{-\frac{2r}{2r+1}}$, it holds that
\begin{equation}\label{ls last iterate capacity independent L2}
\bE_{Z^T}[\|g_{T+1}-f_\rho\|_\rho^2] =\mathcal{O}(T^{-\frac{2r}{2r+1}}\log T).
\end{equation}
And if $r>\frac12,$ there holds
\begin{equation}\label{ls last iterate capacity independent HK}
\bE_{Z^T}[\|g_{T+1}-f_\rho\|_K^2] =\mathcal{O}(T^{-\frac{2r-1}{2r+1}}).
\end{equation}
The convergence rate (\ref{ls last iterate capacity independent HK}) in $\mathcal{H}_K$  is capacity independently optimal, and convergence rate (\ref{ls last iterate capacity independent L2}) in $L^2_{\rho_\X}$ is capacity independently optimal, up to a logarithmic term, in the minimax sense.
For the decreasing step size $\eta_t=\left(1/\nu(2r/2r+1)\right)t^{-\frac{2r}{2r+1}}$ with some $\nu>0$, it shows in \cite{Ying2008} that if Assumption \ref{assumption1} is satisfied with $0<r\le\frac12$, there holds
\begin{equation}\label{ls last iterate capacity independent L2-2}
\bE_{Z^T}[\|g_{T+1}-f_\rho\|_\rho^2] =\mathcal{O}(T^{-\frac{2r}{2r+1}}\log T),
\end{equation}
one can easily see that the best convergence rate of (\ref{ls last iterate capacity independent L2-2}) is $\mathcal{O}(T^{-\frac12}\log T)$ achieved at $r=\frac12$. While the analysis established in \cite{Ying2008} can not lead to faster rates than $\mathcal{O}(T^{-\frac12}\log T)$ when $f_\rho$ has higher regularities, i.e., $r > \frac12$ in Assumption \ref{assumption1}, this is the so-called saturation phenomenon. Whether better rates can be derived with some additional capacity information is an open problem stated in \cite{Ying2008}. If Assumption \ref{assumption1} holds with $r>0$,  our convergence rates in Theorem \ref{thm: convergence rate in L2} is the same as (\ref{ls last iterate capacity independent HK}) in \cite{Ying2008} by selecting  proper scaling parameter $\sigma$. Moreover, if capacity information is available, we obtain more elegant
capacity dependent rates of convergence.  Concretely, under Assumption \ref{assumption1} with $r>\frac12$ and Assumption \ref{assumption2}  with $0<\beta<1,$ Theorem \ref{thm: convergence rate in HK} gives the convergence rate $\mathcal{O}\left(T^{-\frac{2r-1}{2r+\beta}}\right)$ in $\mathcal{H}_K$, which is always faster than $\mathcal{O}\left(T^{-\frac{2r-1}{2r+1}}\right)$ in \cite{Ying2008}.

The other category of convergence analysis for algorithm (\ref{online learning with least square}) provides error bounds involving capacity of the hypothesis space ${\mathcal{H}}_K$, which is more tight than the capacity independent bounds. Capacity dependent convergence rates for the last iterate of algorithm (\ref{online learning with least square}) with decreasing step size are recently derived in \cite{GuoShi2019online}. It shows that if Assumption \ref{assumption1} is satisfied with $r>\frac12$ and Assumption \ref{assumption2} is satisfied with $0<\beta<1$,
\begin{enumerate}
\item[1.] if $\frac12< r\le 1-\frac{\beta}{2},$ and $\eta_t=\eta_1 t^{-\frac{2r}{2r+1}}$ with  $0<\eta_1<\kappa^2,$ there holds
\begin{equation*}
\bE_{Z^T}[\|g_{T+1}-f_\rho\|_\rho^2]=\mathcal{O}\left(T^{-\frac{2r}{2r+1}}\right) ,
\end{equation*}
\item[2.]  if $r>1-\frac{\beta}{2},$ and $\eta_t=\eta_1 t^{-\frac{2-\beta}{3-\beta}}$ with  $0<\eta_1<\kappa^2,$ there holds
\begin{equation*}
\bE_{Z^T}[\|g_{T+1}-f_\rho\|_\rho^2]=\mathcal{O}\left( T^{-\frac{2-\beta}{3-\beta}}\right),
\end{equation*}
\item[3.]  if $r>\frac{1}{2},$ and $\eta_t=\eta_1 t^{-\frac12}$ with $0<\eta_1<{\kappa^2},$ there holds
\begin{equation*}
\bE_{Z^T}[\|g_{T+1}-f_\rho\|_K^2]= \mathcal{O}\left(T^{-\frac{\min\{2r-1,1-\beta\}}{2}}(\log T)^2\right).
\end{equation*}
\end{enumerate}
The above results improved the convergence rates (\ref{ls last iterate capacity independent L2-2}) in $L^2_{\rho_\X}$ and established the first capacity dependent convergence rates in $\mathcal{H}_K$ for the case of decreasing step size. Though these results give an positive answer to the open problem of \cite{Ying2008}, the convergence rates are suboptimal and they are saturated at $r=1-\frac{\beta}{2}$. Our results stated in Theorem \ref{thm: convergence rate in L2} and Theorem \ref{thm: convergence rate in HK} are optimal in the minimax sense and can overcome the saturation phenomenon.
Another recent work \cite{Dieuleveut&Bach2016} provides capacity dependent convergence rates for  the averaging estimator $\bar{g}_T=\frac{1}{T}\sum_{t=1}^T g_t$ of algorithm (\ref{online learning with least square}) depending on polynomial eigendecay condition of $L_K$. Averaging scheme can reduce variance and thus usually leads to better convergence rates (see e.g.,\cite{nemirovski2009robust,rakhlin2012making,yao2010complexity}). If Assumption \ref{assumption1} holds with $r>0$, and the eigenvalues of $L_K$ behave as $c_1 k^{-\frac{1}{\beta}}\le \lambda_k\le c_2 k^{-\frac{1}{\beta}}$ with $0<\beta<1$ and constants $c_1, c_2>0.$ Let $f_\H$ denote the orthogonal projection of $f_\rho$ on ${\cal H}_K$. Particularly, $f_H=f_{\rho}$ if $r\geq \frac12$. For the decreasing step size, it shows in \cite{Dieuleveut&Bach2016} that
\begin{enumerate}
         \item  if $\frac{1}{2}-\frac{\beta}{2}<r\le1-\frac{\beta}{2},$ and $\eta_t=\eta_1 t^{\frac{-2r-\beta+1}{2r+\beta}},$
\begin{equation}\label{decreasing step size averaging rate in L2}
\bE_{Z^{T-1}}\|\bar{g}_T-f_\H\|_\rho^2=\mathcal{O}(T^{-\frac{2r}{2r+\beta}}),
\end{equation}
         \item   if $r>1-\frac{\beta}{2},$ and $\eta_t=\eta_1 t^{-\frac12},$
\begin{equation}\label{decreasing step size averaging rate in L2-2}
\bE_{Z^{T-1}}\|\bar{g}_T-f_\H\|_\rho^2=\mathcal{O}(T^{-(1-\frac{\beta}{2})}).
\end{equation}
\end{enumerate}
For the constant step size $\eta_t=\eta(T),$
\begin{enumerate}
 \item   if $0<r<\frac{1}{2}-\frac{\beta}{2},$ and $\eta_t=\eta_1 $ is a constant,
\begin{equation}\label{constant step size averaging rate in L2-2}
\bE_{Z^{T-1}}\|\bar{g}_T-f_\H\|_\rho^2=\mathcal{O}(T^{-2r}),
\end{equation}
\item  if $r>\frac{1}{2}-\frac{\beta}{2},$ and $\eta_t=\eta(T)=\eta_1 T^{\frac{-2\min\{r,1\}-\beta+1}{2\min\{r,1\}+\beta}}$ with some  constant $\eta_1>0,$
\begin{equation}\label{constant step size averaging rate in L2}
\bE_{Z^{T-1}}\|\bar{g}_T-f_\H\|_\rho^2=\mathcal{O}\left(T^{-\frac{2\min\{r,1\}}{2\min\{r,1\}+\beta}}\right).
\end{equation}
        \end{enumerate}
One can see that the obtained asymptotic convergence rate (\ref{decreasing step size averaging rate in L2}) is optimal when $\frac{1}{2}-\frac{\beta}{2}<r<1-\frac{\beta}{2}$, and (\ref{constant step size averaging rate in L2}) is optimal when $\frac12-\frac{\beta}{2}<r\le 1$,which achieve the mini-max lower bound proved in \cite{Caponnetto2007,Steinwart2009}.  Firstly, the capacity condition in Assumption \ref{assumption2} is more general than the polynomial eigendecay condition adopted in \cite{Dieuleveut&Bach2016} and we do not require the lower bound for the decaying rates.  Secondly, the established convergence rates for $\bar{g}_T$ of  algorithm (\ref{online learning with least square}), whether the step size is chosen to be decreasing or fixed, all suffer from saturation phenomenon, i.e., the convergence rate no longer improves once the regularity of the regression function is beyond certain level. Concretely, if regularity condition in Assumption \ref{assumption1} is satisfied with $r>0$ and Assumption \ref{assumption2} holds with $0<\beta<1$, the convergence rate (\ref{decreasing step size averaging rate in L2}) with decreasing step size is saturated at $r=1-\frac{\beta}{2}$ and ceases to improve as $r>1-\frac{\beta}{2}$, while the convergence rates  (\ref{constant step size averaging rate in L2}) with fixed step size are saturated at $r=1$ and stops getting better when $r>1$. While our convergence analysis established in Theorem \ref{thm: convergence rate in L2} and Theorem \ref{thm: convergence rate in HK}
can eliminate saturation phenomenon and adapt to favorable regularity of $f_{\rho}$ to attain even faster convergence rates.

To obtain these nice results, e.g., the capacity independent optimal rates in $L^2_{\rho_\X}$ and the capacity dependent optimal rates in $\mathcal{H}_K$ taking into account the spectral structure of $L_K$,  the crucial tool developed in our paper is the finer error decomposition techniques. The error decomposition in $L^2_{\rho_\X}$ presented in Proposition \ref{prop: error decomposition in L2}  enable us to utilize the properties of robust losses (see condition \ref{condition1} and condition \ref{condition2}), and to establish a novel framework for error analysis of robust online learning. To achieve strong optimal rates in $\mathcal{H}_K$, we present another error decomposition in Proposition \ref{prop: error decomposition in HK}, which incorporates the trace of the composite operators related to $L_K$. Both of these two error decomposition are first established in this paper and play fundamental roles in our analysis. Through our theoretical analysis, we also demonstrate that the index $\beta$ in trace condition given in assumption \ref{assumption2} is exactly the right parameter to introduce the capacity information of the underlying function space to the setting of online learning.  Actually, the error analysis based on capacity assumption \ref{assumption2} established in our paper also gives a positive answer to an open question proposed by \cite{rosasco2014regularization,Ying2008}, that is, whether we can obtain unsaturated fast and strong convergence under some additional capacity information. In contrast to the effective dimension widely adopted in the error analysis of bath learning (see, e.g., \cite{Caponnetto2007,guo2017learning,lu2020}), trace condition in assumption \ref{assumption2} serves the same role in online learning to establish an important connection between the spectral structure of the operators and the capacity information encoding the crucial properties of the marginal distribution. The capacity dependent analysis of online learning is more involved than that of batch learning. To derive tight convergence rates, we need to choose the step size carefully to settle a bias-variance trade-off based on regularity of $f_{\rho}$ and capacity information of RKHS, requiring us to provide sharp bounds (in the $\mathcal{H}_K$ norm as well as the operator norm) on the items appearing in the error decomposition. With the help of the error decomposition and sharp estimates established in this paper, we obtain minimax optimal rates of (both strong and weak) convergence for robust online learning. Our approach can be extended to study more complex models of online learning such as \cite{chen2021,guo2022capacity,smale2009online}, which we will leave as our future work. Another promising line of research is to apply the shifted loss function proposed in \cite{ChristmannVanMessemSteinwart2009} to design robust online non-parametric learning algorithm. We also consider introducing the $\ell_{\sigma}$ loss to recent empirical studies \cite{feng2022cnn,ZhuLiSun} of deep learning models when outliers or heavy-tailed noise are allowed.

\section{Preliminaries: Error Decomposition and Basic Estimates}\label{section: error decomposition}

In this section, we first introduce an error decomposition for our convergence analysis. In what follows, $\kappa:=\sup_{x\in \cal X}\sqrt{K(x,x)}$ and $C_W:=\sup_{s\in(0,\infty)}\{|W'(s)|\}$ where $W(\cdot)$ is a windowing function satisfying $\eqref{condition1}$ and $\eqref{condition2}$. Recall that the sequence $\{f_t\}_{t\in\mathbf{N}}$ is generated by online algorithm (\ref{algorithm}) with step size $\eta$. Then we have
\begin{align*}
f_{t+1}-f_\rho&=f_t-f_\rho-\eta W'\left(\xi_{t,\sigma}\right)(f_t(x_t)-y_t)K_{x_t}\\
&=f_t-f_\rho-\eta W_+'(0)(f_t(x_t)-y_t)K_{x_t}+\eta \big(W_+'(0)-W'(\xi_{t,\sigma})\big)(f_t(x_t)-y_t)K_{x_t}\\
&=(I-\eta W_+'(0) L_K)(f_t-f_\rho) +\eta W_+'(0)(L_K f_t -f_t(x_t)K_{x_t})+\eta W_+'(0)(y_tK_{x_t}-L_K f_\rho)\\
&\quad +\eta \big(W_+'(0)-W'(\xi_{t,\sigma})\big)(f_t(x_t)-y_t)K_{x_t}\\
&:=(I-\eta W_+'(0) L_K)(f_t-f_\rho) +\eta \B_t+\eta E_{t,\sigma},
\end{align*}
where
\begin{equation*}
\begin{split}
\B_t&=W_+'(0)\left[(L_K f_t -f_t(x_t)K_{x_t})+ (y_tK_{x_t}-L_K f_\rho)\right]\\
&=W_+'(0)\left[L_K(f_t -f_\rho)+ (y_tK_{x_t}-f(x_t)K_{x_t})\right]
\end{split}
\end{equation*}
and
\begin{equation}\label{dfi E t sigma} E_{t,\sigma}=\big(W_+'(0)-W'(\xi_{t,\sigma})\big)(f_t(x_t)-y_t)K_{x_t}.
\end{equation}
By induction, we can decompose $f_{T+1}-f_\rho$ as
\begin{equation}\label{error decomposition}
\begin{split}
f_{T+1}-f_\rho&=-(I-\eta W_+'(0) L_K)^{T}f_\rho+\eta\sum_{t=1}^T(I-\eta W_+'(0)L_K)^{T-t}\B_t\\
&\quad +\eta\sum_{t=1}^T(I-\eta W_+'(0)L_K)^{T-t}E_{\sigma,t}.
\end{split}
\end{equation}
Then we obtain the following error decomposition which is pivotal for convergence analysis in $L_{\rho_\X}^2$. Hereinafter, we use $\|\cdot\|$ to denote the operator norm for operators on $L^2_{\rho_{\cal X}}$ or ${\cal H}_K$, which is specified due to the context. We simply keep the same notion for the two operator norms as $L_K$ is well-defined on both $L^2_{\rho_{\cal X}}$ and ${\cal H}_K$.

\begin{proposition}\label{prop: error decomposition in L2}
Let $\{f_t\}_{t=1}^{T+1}$ be defined by (\ref{algorithm}). Then
\begin{align*}
\bE_{Z^T}\left[\|f_{T+1}-f_\rho\|_\rho^2\right]&\leq2\left\|(I-\eta W_+'(0) L_K)^{T}f_\rho\right\|_\rho^2+2\eta^2\bE_{Z^T}\left[\left\|\sum_{t=1}^T(I-\eta W_+'(0)L_K)^{T-t}E_{\sigma,t}\right\|_\rho^2\right]\\
&\quad+2\eta^2(\kappa W_+'(0))^2\sum_{t=1}^T\left\|L_K^{\frac12}(I-\eta W_+'(0)L_K)^{T-t}\right\|^2\bE_{Z^{t-1}}[\E(f_t)],
\end{align*}
where $E_{\sigma,t}$ is defined by \eqref{dfi E t sigma} and
\begin{equation}\label{definition of Ef}
\E(f):=\int_{\X\times\Y} (f(x)-y)d\rho, \quad \forall f: \X \to \Y \mbox{ is measurable}.
\end{equation}
\end{proposition}
\begin{proof}
By the decomposition (\ref{error decomposition}), we have
\begin{align*}
&\bE_{Z^T}\left[\|f_{T+1}-f_\rho\|_\rho^2\right]\\
&=\bE_{Z^T}\left[\left\|-(I-\eta W_+'(0) L_K)^{T}f_\rho+\eta\sum_{t=1}^T(I-\eta W_+'(0)L_K)^{T-t}\B_t+\eta\sum_{t=1}^T(I-\eta W_+'(0)L_K)^{T-t}E_{\sigma,t}\right\|_\rho^2\right]\\
&\le 2\bE_{Z^T}\left[\left\|-(I-\eta W_+'(0) L_K)^{T}f_\rho+\eta\sum_{t=1}^T(I-\eta W_+'(0)L_K)^{T-t}\B_t\right\|_\rho^2\right]\\
&\quad +2\bE_{Z^T}\left[\left\|\eta\sum_{t=1}^T(I-\eta W_+'(0)L_K)^{T-t}E_{\sigma,t}\right\|_\rho^2\right]\\
&= 2\left\|(I-\eta W_+'(0) L_K)^{T}f_\rho\right\|_\rho^2+2\eta^2\bE_{Z^T}\left[\left\|\sum_{t=1}^T(I-\eta W_+'(0)L_K)^{T-t}\B_t\right\|_\rho^2\right]\\
&\quad +2\eta^2\bE_{Z^T}\left[\left\|\sum_{t=1}^T(I-\eta W_+'(0)L_K)^{T-t}E_{\sigma,t}\right\|_\rho^2\right].
\end{align*}
The last equality holds since $f_t$ depends only on $z_1,\cdots,z_{t-1},$  and $\bE_{z_t}[\B_t]=0,$ it follows that
\begin{align*}
&\bE_{Z^T}\langle -(I-\eta W_+'(0) L_K)^{T}f_\rho, \eta\sum_{t=1}^T(I-\eta W_+'(0)L_K)^{T-t}\B_t\rangle_\rho\\
&=\langle -(I-\eta W_+'(0) L_K)^{T}f_\rho, \eta\sum_{t=1}^T(I-\eta W_+'(0)L_K)^{T-t}\bE_{z_1,\cdots,z_{t-1}}\bE_{z_t}[\B_t]\rangle_\rho=0
\end{align*}
Furthermore, for the term $\eta^2\bE_{Z^T}\left[\left\|\sum_{t=1}^T(I-\eta W_+'(0)L_K)^{T-t}\B_t\right\|_\rho^2\right]$, we have
\begin{align*}
&\eta^2\bE_{Z^T}\left[\left\|\sum_{t=1}^T(I-\eta W_+'(0)L_K)^{T-t}\B_t\right\|_\rho^2\right]\\
&=\eta^2\sum_{t=1}^T\bE_{Z^t}\left[\left\|(I-\eta W_+'(0)L_K)^{T-t}\B_t\right\|_\rho^2\right]\\
&\quad +\eta^2\sum_{t=1}^T\sum_{k\neq t}\bE_{Z^t}\langle(I-\eta W_+'(0)L_K)^{T-t}\B_t ,(I-\eta W_+'(0)L_K)^{T-k}\B_k\rangle_\rho\\
&=\eta^2\sum_{t=1}^T\bE_{Z^t}\left[\left\|(I-\eta W_+'(0)L_K)^{T-t}\B_t\right\|_\rho^2\right],
\end{align*}
the last equality holds since for $t>k$
$$\bE_{z_1,\cdots,z_{t-1}}\bE_{z_t}\langle(I-\eta W_+'(0)L_K)^{T-t}\B_t ,(I-\eta W_+'(0)L_K)^{T-k}\B_k\rangle_\rho=0,$$
and also for $t<k$
$$\bE_{z_1,\cdots,z_{k-1}}\bE_{z_k}\langle(I-\eta W_+'(0)L_K)^{T-t}\B_t ,(I-\eta W_+'(0)L_K)^{T-k}\B_k\rangle_\rho=0.$$
Moreover, recall that $\mathcal{B}_t= W_+'(0)\left[(y_t-f_t(x_t))K_{x_t}-L_K(f_\rho-f_t)\right]$ for $1 \leq t \leq T$. We have $\bE_{z_t}\left[\mathcal{B}_t\right]=0$ and
\begin{align*}
\bE_{z_t}\left[\|\mathcal{B}_t\|_K^2\right]&\le (W_+'(0))^2\bE_{z_t}\left[\|(y_t-f_t(x_t))K_{x_t}\|_K^2\right]\\
&=(W_+'(0))^2\bE_{z_t}\left[(y_t-f_t(x_t))^2K(x_t,x_t)\right]\\
&\le (\kappa W_+'(0))^2\int_{\X\times \Y}(f_t(x)-y)d\rho:=(\kappa W_+'(0))^2\E(f_t).
\end{align*}
It then follows that
\begin{align*}
&\eta^2\sum_{t=1}^T\bE_{Z^T}\left[\left\|(I-\eta W_+'(0)L_K)^{T-t}\B_t\right\|_\rho^2\right]\\
&=\eta^2\sum_{t=1}^T\bE_{Z^t}\left[\left\|(I-\eta W_+'(0)L_K)^{T-t}L_K^{\frac12}L_K^{-\frac12}\B_t\right\|_\rho^2\right]\\
&\le \eta^2\sum_{t=1}^T\left\|(I-\eta W_+'(0)L_K)^{T-t}L_K^{\frac12}\right\|^2 \bE_{Z^{t}}\left[\|\B_t\|_K^2\right]\\
&\le \eta^2 (\kappa W_+'(0))^2 \sum_{t=1}^T\left\|(I-\eta W_+'(0)L_K)^{T-t}L_K^{\frac12}\right\|^2 \bE_{Z^{t-1}}\left[\E(f_t)\right].
\end{align*}
This completes the proof.
\end{proof}

Next, we present some basic estimates. First we prove the following lemma.
\begin{lemma}\label{lem: key lemma}
Let $\alpha>0$, $s>0$, and $\eta<\frac{1}{\kappa^2 W_+'(0)}$. Then
\begin{align}\label{key lemma}
\left\|L_K^{\alpha}(I-\eta W_+'(0)L_K)^{s}\right\|\le \left(\frac{\alpha}{eW_+'(0)}\right)^{\alpha}(\eta s)^{-\alpha}.
\end{align}
\end{lemma}
\begin{proof}
One can easily check that $1-u\le \exp{(-u)},\forall u\ge 0$. Then it follows that
\begin{align*}
\left\|L_K^{\alpha}(I-\eta W_+'(0)L_K)^{s}\right\|
&\le \sup_{x>0}x^\alpha(1-\eta W_+'(0)x)^s\\
&\le \sup_{x>0}x^\alpha\exp{(-\eta W_+'(0)xs)}\le \left(\frac{\alpha}{eW_+'(0)}\right)^{\alpha}(\eta s)^{-\alpha},
\end{align*}
the last inequality holds since function $g(x)=x^\alpha\exp{(-\eta W_+'(0)xs)}$ is minimized at $x=\frac{\alpha}{W_+'(0)\eta s}.$ Thus we completes the proof.
\end{proof}
Recall that $|y|\leq M$ almost surely. We establish the following bound for the sequence $\{f_t\}_{t=1}^{T+1}$.
\begin{proposition}\label{prop: bound for ft}
Let $\{f_t\}_{t=1}^{T+1}$ be defined by algorithm (\ref{algorithm}). If the step size $\eta$ satisfies $\eta\le \frac{1}{\kappa^2C_W}$, then
\begin{align}\label{bound for ft}
\|f_{t+1}\|_K^2\leq\ M^2C_W \eta t.
\end{align}
\end{proposition}
\begin{proof}
We prove $(\ref{bound for ft})$ by induction. The initial function $f_1=0$ obviously satisfies the inequality $(\ref{bound for ft})$. For $f_2$, the iteration (\ref{algorithm}) indicates that
\begin{equation*}
\|f_2\|_K=\left\|-{\eta} W'\left(\xi_{1,\sigma}\right)(-y_i)K_{x_i}\right\|_K\le \kappa  M C_W \eta \le M\sqrt{C_W\eta}\le M\sqrt{2C_W\eta},
\end{equation*}
where the last inequality holds due to the choice of $\eta$. Thus we show that $f_2$ also satisfies (\ref{bound for ft}). When $t\ge 2$, we rewrite (\ref{algorithm}) as $f_{t+1}=f_t-\eta H_t$ with $H_t=W'\left(\xi_{t,\sigma}\right)(f_t(x_t)-y_t)K_{x_t}.$
Then
\begin{equation}\label{rewri}
\begin{split}
\|f_{t+1}\|_K^2&=\|f_{t}\|_K^2-2\eta\langle f_t,H_t\rangle_K+\eta^2\|H_t\|_K^2 \\
&=\|f_{t}\|_K^2-2 \eta W'\left(\xi_{t,\sigma}\right)(f_t(x_t)-y_t)f_t(x_t)+\eta^2\|H_t\|_K^2,
\end{split}
\end{equation}
and one can easily see that
\begin{align*}
\|H_t\|_K^2\leq \kappa^2 \left(W'\left(\xi_{t,\sigma}\right)\right)^2(f_t(x_t)-y_t)^2.
\end{align*}
Thus by (\ref{rewri}), $\|f_{t+1}\|^2_K$ can be bounded by
\begin{align}\label{f t+1}
\|f_{t}\|^2_K+\eta\left[\eta\kappa^2W'\left(\xi_{t,\sigma}\right)
(f_t(x_t)-y_t)^2-2(f_t(x_t)-y_t)f_t(x_t)\right]W'\left(\xi_{t,\sigma}\right).
\end{align}
For each $t$, we further have
\begin{equation*}
\begin{split}
&\eta\kappa^2W'\left(\xi_{t,\sigma}\right)(f_t(x_t)-y_t)^2-2(f_t(x_t)-y_t)f_t(x_t)\\
&=\left(\eta\kappa^2W'\left(\xi_{t,\sigma}\right)-2\right)\left((f_t(x_t)-y_t)-
\frac{y_t}{\eta\kappa^2W'\left(\xi_{t,\sigma}\right)-2}\right)^2
+\frac{y_t^2}{2-\eta\kappa^2W'\left(\xi_{t,\sigma}\right)}.
\end{split}
\end{equation*}
Since $W'\left(\xi_{t,\sigma}\right)\le C_W$ and $\eta \kappa^2 C_W\leq 1$, it follows that $\eta\kappa^2W'\left(\xi_{t,\sigma}\right)-2<0$ and $2-\eta\kappa^2W'\left(\xi_{t,\sigma}\right)> 1.$ Moreover, since $|y|\leq M$, there holds
\begin{equation*}
\eta\kappa^2W'\left(\xi_{t,\sigma}\right)(f_t(x_t)-y_t)^2-2(f_t(x_t)-y_t)f_t(x_t)
\le \frac{y_t^2}{2-\eta\kappa^2W'\left(\xi_{t,\sigma}\right)}\le  M^2.
\end{equation*}
Putting the above bound and the induction assumption $\|f_{t}\|^2_K\le M^2C_W (t-1)\eta$ into (\ref{f t+1}) yields
\begin{equation*}
 \|f_{t+1}\|_K^2\leq\|f_{t}\|_K^2+\eta M^2C_W\leq M^2C_W (t-1)\eta+\eta M^2C_W \leq M^2C_W t\eta.
\end{equation*}
Therefore, the proof is completed.
\end{proof}
We also establish a uniform bound of $\|E_{t,\sigma}\|_K$ for $1\le t\le T$, which will play a crucial role in our convergence analysis. Recall that the windowing function $W(\cdot)$ satisfies condition (\ref{condition2}) with some constants $c_p>0$ and $p>0$.
\begin{proposition}\label{prop: bound for Esigma}
Let $E_{\sigma,t}$ be defined by (\ref{dfi E t sigma}) with $1\le t\le T.$  Then
\begin{equation}\label{bound for Esigma}
\|E_{t,\sigma}\|_K\le \kappa c_p\left(M+ \kappa M\sqrt{C_W }\right)^{2p+1} \frac{(\eta T)^{p+\frac12}}{\sigma^{2p}}.
\end{equation}
\end{proposition}
\begin{proof}
Recall that
$E_{t,\sigma}=\left(W_+'(0)-W'\left({\frac{(y_t-f_t(x_t))^2}{\si^2}}\right)\right)(f_t(x_t)-y_t)K_{x_t}.$
Since the function $W(\cdot)$ satisfies condition (\ref{condition2}) with constants $c_p>0$ and $p>0,$ we have
\begin{equation*}
\left|W_+'(0)-W'\left({\frac{(y_t-f_t(x_j))^2}{\si^2}}\right)\right|\le c_p\left(\frac{(y_t-f_t(x_t))^2}{\sigma^2}\right)^p\le c_p\left(\frac{M+\kappa\|f_t\|_K}{\sigma}\right)^{2p}.
\end{equation*}
Combining the above estimate with the bound (\ref{bound for ft}) for $\|f_t\|_K$ in Proposition \ref{prop: bound for ft}, we get
\begin{equation*}\label{E i sigma bound}
\|E_{t,\sigma}\|_K\le \frac{\kappa c_p(M+\kappa\|f_t\|_K)^{2p+1}}{\sigma^{2p}}\le \kappa c_p\left(M+ \kappa M\sqrt{C_W \eta (t-1)}\right)^{2p+1} \frac{1}{\sigma^{2p}}.
\end{equation*}
Hence, the uniform bound \eqref{bound for Esigma} holds true for all $1\le t\le T$. Thus we complete the proof.
\end{proof}
The following bound for the second term in Proposition \ref{prop: error decomposition in L2} is an immediate consequence of Proposition \ref{prop: bound for Esigma}.
\begin{proposition}\label{prop: the third term in L2}
Let $E_{\sigma,t}$ be defined by (\ref{dfi E t sigma}) with $1\le t\le T.$ Then
\begin{align}\label{the third term in L2}
2\eta^2\bE_{Z^T}\left[\left\|\sum_{t=1}^T(I-\eta W_+'(0)L_K)^{T-t}E_{\sigma,t}\right\|_\rho^2\right]\le C_1\frac{(\eta T)^{2p+2}}{\sigma^{4p}},
\end{align}
where $C_1= 2\kappa^2 c_p^2\left(M+ \kappa M\sqrt{C_W }\right)^{4p+2}  \left( \kappa+ \left(\frac{2}{eW_+'(0)}\right)^{\frac12} \right)^2$.
\end{proposition}
\begin{proof} By the relationship (\ref{normrelation2}) between $\|\cdot\|_\rho$ and $\|\cdot\|_K$, Lemma \ref{lem: key lemma} with $\alpha=\frac12$ and $s=t$, and the uniform bound (\ref{bound for Esigma}) for $\left\|E_{\sigma,t}\right\|_K^2,$ we have
\begin{align*}
&2\eta^2\bE_{Z^T}\left[\left\|\sum_{t=1}^T(I-\eta W_+'(0)L_K)^{T-t}E_{\sigma,t}\right\|_\rho^2\right]\\
&=2\eta^2\bE_{Z^T}\left[\left\|L_K^{\frac12}\sum_{t=1}^T(I-\eta W_+'(0)L_K)^{T-t}E_{\sigma,t}\right\|_K^2\right]\\
&\le  2\eta^2\max_{1\le t\le T} \left\|E_{\sigma,t}\right\|_K^2 \times \left(\sum_{t=0}^{T-1} \left\|L_K^{\frac12}(I-\eta W_+'(0)L_K)^{t}\right\|\right)^2 \\
&\le  2\eta^2\kappa^2 c_p^2\left(M+ \kappa M\sqrt{C_W }\right)^{4p+2} \frac{(\eta T)^{2p+1}}{\sigma^{4p}} \times \left( \kappa+\sum_{t=1}^{T-1} \left(\frac{1}{2eW_+'(0)}\right)^{\frac12}(\eta t)^{-\frac12}\right)^2\\
& \le 2\eta^2\kappa^2 c_p^2\left(M+ \kappa M\sqrt{C_W }\right)^{4p+2} \frac{(\eta T)^{2p+1}}{\sigma^{4p}} \times  \left( \kappa+ 2\left(\frac{1}{2eW_+'(0)}\right)^{\frac12} \sqrt{\frac{T}{\eta}}\right)^2\\
&\le   2\kappa^2 c_p^2\left(M+ \kappa M\sqrt{C_W }\right)^{4p+2}  \left( \kappa+ \left(\frac{2}{eW_+'(0)}\right)^{\frac12} \right)^2 \frac{(\eta T)^{2p+2}}{\sigma^{4p}}.
\end{align*}
This finishes the proof.
\end{proof}
The following uniform bound for $\bE_{Z^t}[\E(f_{t+1})]$ is also important for our analysis.
\begin{proposition}\label{prop: bound for E(ft)}
If $\eta$ satisfies
\begin{equation}\label{eta condition 2}
\eta\le\frac{1}{\left( \frac{1}{e}+2\kappa^2W_{+}'(0)\right)^2\log T},
\end{equation}
then for $1\le t\le T$, there holds
\begin{align}\label{bound for E(ft)}
\bE_{Z^t}[\E(f_{t+1})]\le 2\E(f_\rho)+ 4\|f_\rho\|_\rho^2+2C_1\frac{(\eta T)^{2p+2}}{\sigma^{4p}}
\end{align} where $C_1$ is given in Proposition \ref{prop: the third term in L2}.
\end{proposition}
\begin{proof}
We bound $\bE_{Z^t}\left[\E(f_{t+1})\right]$ by induction. Due to Proposition \ref{prop: error decomposition in L2}, Lemma \ref{lem: key lemma} with $\alpha=\frac12$ and $s=t-i$, and Proposition \ref{prop: the third term in L2}, $\bE_{Z^t}\left[\|f_{t+1}-f_\rho\|_\rho^2\right]$ can be bounded as
\begin{align*}
&\bE_{Z^t}\left[\|f_{t+1}-f_\rho\|_\rho^2\right]\\&\leq 2\left\|(I-\eta W_+'(0) L_K)^{t}f_\rho\right\|_\rho^2++2\eta^2\bE_{Z^t}\left[\left\|\sum_{i=1}^t(I-\eta W_+'(0)L_K)^{t-i}E_{\sigma,i}\right\|_\rho^2\right]\\
&\quad+2\eta^2(\kappa W_+'(0))^2\sum_{i=1}^t\left\|(I-\eta W_+'(0)L_K)^{t-i}L_K^{\frac12}\right\|^2\bE_{Z^{i-1}}\left[\E(f_i)\right]\\
&\le 2\|f_\rho\|_\rho^2+ C_1\frac{(\eta T)^{2p+2}}{\sigma^{4p}}+2\eta^2 (\kappa W_+'(0))^2  \left( \kappa^2+\frac{1}{2eW_+'(0)}\sum_{i=1}^{t-1}(\eta (t-i))^{-1}\right) \sup_{1 \leq i \leq t}\bE_{Z^{i-1}}\left[\E(f_i)\right]\\
&\le 2\|f_\rho\|_\rho^2+C_1\frac{(\eta T)^{2p+2}}{\sigma^{4p}}+2\eta   \left( \kappa^2W_+'(0)+\frac{1}{2e}\right)^2 \sup_{1 \leq i \leq t}\log t \bE_{Z^{i-1}}\left[\E(f_i)\right].
\end{align*}
Then if bound \eqref{bound for E(ft)} is ture for  $\bE_{Z^{i-1}}\left[\E(f_i)\right]$, combining with the condition (\ref{eta condition 2}) on $\eta$ and the relation  $\E(f_{t+1})=\E(f_\rho)+\E(f_{t+1})-\E(f_\rho)=\E(f_\rho)+\|f_{t+1}-f_\rho\|_\rho^2$, we obtain
\begin{align*}
\bE_{Z^t}(\E(f_{t+1}))
&\le \E(f_\rho)+ 2\|f_\rho\|_\rho^2+ \frac{1}{2}\left(2\E(f_\rho)+ 4\|f_\rho\|_\rho^2+2C_1\frac{(\eta T)^{2p+2}}{\sigma^{4p}}\right)
+C_1\frac{(\eta T)^{2p+2}}{\sigma^{4p}}\\
&=2\E(f_\rho)+ 4\|f_\rho\|_\rho^2+2C_1\frac{(\eta T)^{2p+2}}{\sigma^{4p}}.
\end{align*}
This finishes our proof.
\end{proof}

\section{Convergence Analysis}\label{section: proof of main results}
In this section, we give the proofs of Theorem \ref{thm: convergence rate in L2} and Theorem \ref{thm: convergence rate in HK}.
\subsection{Convergence in $L_{\rho_\X}^2$}\label{subsection: convergence in L2}
This subsection is devoted to the proof of Theorem \ref{thm: convergence rate in L2}, which provides the convergence rates in $L_{\rho_\X}^2$.

{\noindent {\bf Proof of Theorem \ref{thm: convergence rate in L2}}.}
Due to the error decomposition in Proposition \ref{prop: error decomposition in L2}, we only need to estimate the three terms appeared in the upper bound of $\bE_{Z^T}\left[\|f_{T+1}-f_\rho\|_\rho^2\right]$ respectively. As an estimate for the second term is given by \eqref{the third term in L2} of Proposition \ref{prop: the third term in L2}, we turn to bound the remaining two terms.

For the first term, since the regularity condition (\ref{regularity condition}) holds with $r>0$, i.e., $f_\rho=L_K^r g_\rho$ with $g_\rho\in L_{\rho_\X}^2$ and $r>0$, then by Lemma \ref{lem: key lemma} with $\alpha=r$ and $s=T$, we have
\begin{equation}\label{the first term in L2}
\begin{split}
&2\left\|(I-\eta W_+'(0) L_K)^{T}f_\rho\right\|_\rho^2\\
&\le 2\left\|(I-\eta W_+'(0) L_K)^{T}L_K^r\right\|^2 \left\| g_\rho\right\|_\rho^2\\
&\le 2\left( \frac{r}{eW_+'(0)}\right)^{2r}\left\| g_\rho\right\|_\rho^2(\eta T)^{-2r}.
\end{split}
\end{equation}

For the third term, by Lemma \ref{lem: key lemma} with $\alpha=\frac12$ and $s=T-t$, and the uniform bound (\ref{bound for E(ft)}) for $\bE_{Z^{t-1}}\left[\E(f_t)\right]$ in Proposition \ref{prop: bound for E(ft)}, we have
\begin{align*}
&2\eta^2(\kappa W_+'(0))^2 \sum_{t=1}^T\left\|(I-\eta W_+'(0)L_K)^{T-t}L_K^{\frac12}\right\|^2 \bE_{Z^{t-1}}\left[\E(f_t)\right]\\
&\le 2\eta^2(\kappa W_+'(0))^2 \max_{1\le t\le T}\bE_{Z^{t-1}}\left[\E(f_t)\right]\left(\kappa^2+\sum_{t=1}^{T-1}\left\|(I-\eta W_+'(0)L_K)^{T-t}L_K^{\frac12}\right\|^2 \right)\\
&\le2\eta^2 (\kappa W_+'(0))^2 \left(2\E(f_\rho)+ 4\|f_\rho\|_\rho^2+2C_1\frac{(\eta T)^{2p+2}}{\sigma^{4p}}\right) \left(\kappa^2+ \frac{1}{2eW_+'(0)} \sum_{t=1}^{T-1}\frac{1}{\eta (T-t)}\right)\\
&\le2(\kappa W_+'(0))^2 \left(2\E(f_\rho)+ 4\|f_\rho\|_\rho^2+2C_1\frac{(\eta T)^{2p+2}}{\sigma^{4p}}\right) \left(\kappa^2+ \frac{1}{2eW_+'(0)} \right)\eta \log T.
\end{align*}
Then the third term can be bounded as
\begin{equation}\label{the second term in L2}
\begin{split}
&2\eta^2(\kappa W_+'(0))^2 \sum_{t=1}^T\left\|(I-\eta W_+'(0)L_K)^{T-t}L_K^{\frac12}\right\|^2 \bE_{Z^{t-1}}\left[\E(f_t)\right]\\
&\le 4(\kappa W_+'(0))^2 \left(\E(f_\rho)+ 2\|f_\rho\|_\rho^2+C_1\right)\left( \kappa^2+\frac{1}{2eW_+'(0)}\right)\left( 1+(\eta T)^{2p+2}\sigma^{-4p}\right) \eta\log T.
\end{split}
\end{equation}
Putting the estimates (\ref{the first term in L2}), (\ref{the second term in L2}) and (\ref{the third term in L2}) back into proposition \ref{prop: error decomposition in L2}, and by taking $\eta=\frac1{\eta_0} T^{-\frac{2r}{2r+1}}$ yields
\begin{align*}
&\bE_{Z^T}\left[\|f_{T+1}-f_\rho\|_\rho^2\right]\\
&\le 2\left( \frac{r}{eW_+'(0)}\right)^{2r}\left\| g_\rho\right\|_\rho^2(\eta T)^{-2r}+ C_1\frac{(\eta T)^{2p+2}}{\sigma^{4p}}\\
&\quad+4(\kappa W_+'(0))^2 \left(\E(f_\rho)+ 2\|f_\rho\|_\rho^2+C_1\right)\left( \kappa^2+ \frac{1}{2eW_+'(0)}\right)^2\left( 1+(\eta T)^{2p+2}\sigma^{-4p}\right) \eta\log T\\
&\le C\max\left\{T^{-\frac{2r}{2r+1}}\log T T^{\frac{2p+2}{2r+1}}\sigma^{-4p}\right\},
\end{align*}
where
\begin{align*}
	&&C=2\left( \frac{r}{eW_+'(0)}\right)^{2r}\left\| g_\rho\right\|_\rho^2\eta_0^{2r}
	+4\eta_0(\kappa W_+'(0))^2 \left(\E(f_\rho)
	+ 2\|f_\rho\|_\rho^2+C_1\right)\\
	&&\times\left( \kappa^2+\frac{1}{2eW_+'(0)}\right)^2\left( 1+\eta_0 ^{-(2p+2)}\right)
	+ C_1\eta_0 ^{-(2p+2)}.
\end{align*}
The proof is completed. \qed

\subsection{Capacity Dependent Analysis in $\H_K$}
In this section, we consider the convergence of algorithm (\ref{algorithm}) in $\H_K,$ develop a capacity dependent analysis. We show that the algorithm (\ref{algorithm}) can achieve the optimal learning rates in the minimax sense in $\H_K$. Before proving the main result, we establish the following error decomposition which is different from the one in $L_{\rho_X}^2$.
\begin{proposition}\label{prop: error decomposition in HK}
Let $\{f_t\}_{t=1}^T$ be defined by (\ref{algorithm}). Then
\begin{align*}
\bE_{Z^T}\left[\|f_{T+1}-f_\rho\|_K^2\right]&\leq 2\left\|(I-\eta W_+'(0) L_K)^{T}f_\rho\right\|_K^2+2\eta^2\bE_{Z^T}\left[\left\|\sum_{t=1}^T(I-\eta W_+'(0)L_K)^{T-t}E_{\sigma,t}\right\|_K^2\right]\\
&\quad+4\eta^2(W_+'(0))^2\sum_{t=1}^T\left(\kappa^2\bE_{Z^{t-1}}[\|f_t\|_K^2]+M^2\right){\rm Tr}\left(L_K(I-\eta W_+'(0) L_K)^{2(T-t)}\right),
\end{align*} where $E_{\sigma,t}$ is defined by \eqref{dfi E t sigma}.
\end{proposition}
\begin{proof}
First, by the error decomposition (\ref{error decomposition}) and proof of  Proposition \ref{prop: error decomposition in L2}, we get
\begin{align*}
\bE_{Z^T}\left[\|f_{T+1}-f_\rho\|_K^2\right]&\le 2\left\|(I-\eta W_+'(0) L_K)^{T}f_\rho\right\|_K^2+2\eta^2\bE_{Z^T}\left[\left\|\sum_{t=1}^T(I-\eta W_+'(0)L_K)^{T-t}E_{\sigma,t}\right\|_K^2\right]\\&\quad+2\eta^2\sum_{t=1}^T\bE_{Z^T}\left[\left\|(I-\eta W_+'(0)L_K)^{T-t}\B_t\right\|_K^2\right].
\end{align*}
For the third term $2\eta^2\sum_{t=1}^T\bE\left[\left\|(I-\eta W_+'(0)L_K)^{T-t}\B_t\right\|_K^2\right],$ recall that $\B_t=W_+'(0)(L_K(f_t -f_\rho)+ (y_tK_{x_t}-f(x_t)K_{x_t})),$ we have
\begin{align*}
&\bE_{z_t}\left[\left\|(I-\eta W_+'(0)L_K)^{T-t}\B_t\right\|_K^2\right]\le (W_+'(0))^2\bE_{z_t}\left[\left\|(I-\eta W_+'(0)L_K)^{T-t}(y_t-f_t(x_t))K_{x_t}\right\|_K^2\right]\\
&\le(W_+'(0))^2\left(2\kappa^2\|f_t\|_K^2+2M^2\right)\bE_{z_t}\left[\left\|(I-\eta W_+'(0)L_K)^{T-t}K_{x_t}\right\|_K^2\right]\\
&=(W_+'(0))^2\left(2\kappa^2\|f_t\|_K^2+2M^2\right){\rm Tr}\left(L_K(I-\eta W_+'(0) L_K)^{2(T-t)}\right).
\end{align*}
Therefore, $2\eta^2\sum_{t=1}^T\bE\left[\left\|(I-\eta W_+'(0)L_K)^{T-t}\B_t\right\|_K^2\right]$ can be bounded by
\begin{align*}
2\eta^2(W_+'(0))^2\sum_{t=1}^T\left(2\kappa^2\bE_{Z^{t-1}}[\|f_t\|_K^2]+2M^2\right){\rm Tr}\left(L_K(I-\eta W_+'(0) L_K)^{2(T-t)}\right).
\end{align*}
This completes the proof of Proposition \ref{prop: error decomposition in HK}.
\end{proof}
{\noindent To prove our main results, we need the following bound for $\bE_{Z^{t-1}}\|f_t\|_K^2$.}
\begin{lemma}\label{lemma: ft bound}
Let $\{f_t\}_{t=1}^T$ be defined by \eqref{algorithm}. Then
\begin{align}\label{expectation bound for ft}
\bE_{\Z^{t-1}}[\|f_t\|_K^2]\le  C_2(1+(\eta T)^{2p+3}{\sigma^{-4p}}), \quad \forall 1\le t\le T,
\end{align}
where $$C_2=6\|f_\rho\|_K^2+ 8(\kappa W_+'(0))^2\left(\E(f_\rho)+ 2\|f_\rho\|_\rho^2+C_1\right)+2\kappa^2 c_p^2\left(M+ \kappa M\sqrt{C_W }\right)^{4p+2}$$ and $C_1$ is given in Proposition \ref{prop: the third term in L2}.
\end{lemma}
\begin{proof} Recall that $\mathcal{B}_i= W_+'(0)\left((y_i-f_i(x_i))K_{x_i}-L_K(f_\rho-f_i)\right)$ for $1 \leq i \leq t$. Then $\bE_{z_i|Z^{i-1}}\left[\mathcal{B}_i\right]=0$ and
\begin{align*}
\bE_{z_i|Z^{i-1}}\left[\|\mathcal{B}_i\|_K^2\right]&\le (W_+'(0))^2\bE_{z_i|Z^{i-1}}\left[\|(y_i-f_i(x_i))K_{x_i}\|_K^2\right]\\
&=(W_+'(0))^2\bE_{z_i|Z^{i-1}}\left[(y_i-f_i(x_i))^2K(x_i,x_i)\right]\le (\kappa W_+'(0))^2\E(f_i),
 \end{align*}
where we use the fact that $f_i$ is a random variable independent of $z_i$.  Combining with the definition of $f_t$ given by (\ref{algorithm}), we have
\begin{align*}
\bE_{Z^t}\left[\|f_{t+1}-f_\rho\|_K^2\right]&\le 2\left\|(I-\eta W_+'(0) L_K)^{t}f_\rho\right\|_K^2+2\eta^2\sum_{i=1}^t\bE_{Z^i}\left[\left\|(I-\eta W_+'(0)L_K)^{t-i}\B_i\right\|_K^2\right]\\
&\quad +2\eta^2\bE_{Z^t}\left[\left\|\sum_{i=1}^t(I-\eta W_+'(0)L_K)^{t-i}E_{\sigma,i}\right\|_K^2\right]\\
&\le 2\|f_\rho\|_K^2+2\eta^2\sum_{i=1}^t\left\|(I-\eta W_+'(0)L_K)^{t-i}\right\|\bE_{Z^i}\left[\left\|\B_i\right\|_K^2\right]\\
&\quad +2\eta^2\bE_{Z^t}\left[\left\|\sum_{i=1}^t(I-\eta W_+'(0)L_K)^{t-i}E_{\sigma,i}\right\|_K^2\right]\\
&\le 2\|f_\rho\|_K^2+2\eta^2(\kappa W_+'(0))^2\sum_{i=1}^t\bE_{Z^{i-1}}\left[\E(f_i)\right]
+2\eta^2\bE_{Z^{t}}\left[\sum_{i=1}^t\left\|E_{\sigma,i}\right\|_K\right]^2.
\end{align*}
Then putting the bounds (\ref{bound for E(ft)}) and (\ref{bound for Esigma}) for $\bE_{Z^{i-1}}\left[\E(f_i)\right]$ and $\left\|E_{\sigma,i}\right\|_K$ back into the above inequality yields
\begin{align*}
\bE_{Z^t}\left[\|f_{t+1}-f_\rho\|_K^2\right]
&\le 2\|f_\rho\|_K^2+ 2\eta^2T(\kappa W_+'(0))^2\left(2\E(f_\rho)+ 4\|f_\rho\|_\rho^2+2C_1(\eta T)^{2p+2}\sigma^{-4p}\right) \\
&\quad +2\kappa^2 c_p^2\left(M+ \kappa M\sqrt{C_W }\right)^{4p+2} (\eta T)^{2p+3}{\sigma^{-4p}}.
\end{align*} Finally we obtain the desired bound \eqref{expectation bound for ft} by the relation
$$\bE_{Z^t}\left[\|f_{t+1}\|_K^2\right]\le 2\bE_{Z^t}\left[\|f_{t+1}-f_\rho\|_K^2\right]+ 2\|f_\rho\|_K^2.$$
This completes the proof.
\end{proof}
Now we are in a position to prove the convergence rates in $\H_K.$

{\noindent {\bf Proof of Theorem \ref{thm: convergence rate in HK}}:}
Similar as the proof of Theorem \ref{thm: convergence rate in L2}, due to  Proposition \ref{prop: error decomposition in HK}, we need to estimate the three terms appeared in the upper bound of $\bE_{Z^T}\left[\|f_{T+1}-f_\rho\|_K^2\right]$ respectively.

For the first term, since the target function $f_\rho$ satisfies the regularity condition (\ref{regularity condition}) with $r> \frac12$, i.e., $f_\rho=L_K^{r}g_\rho$ with $g_\rho\in L_{\rho_X}^2$ and $r>\frac12,$ then by Lemma \ref{lem: key lemma} with $\alpha=r-\frac12$ and $s=T,$ we have
\begin{equation}\label{the first term in HK}
\begin{split}
&2\left\|(I-\eta W_+'(0) L_K)^{T}f_\rho\right\|_K^2\\
&=2\left\|(I-\eta W_+'(0) L_K)^{T}L_K^{r-\frac12} L_K^{\frac12}g_\rho\right\|_K^2\\
&\le 2\left\|(I-\eta W_+'(0) L_K)^{T}L_K^{r-\frac12}\right\|^2\left\|g_\rho\right\|_\rho^2\\
&\le 2\left( \frac{r-\frac12}{eW_+'(0)}\right)^{2(r-\frac12)}\left\| g_\rho\right\|_\rho^2(\eta T)^{2(r-\frac12)}).
\end{split}
\end{equation}

The property of trace shows that if $A$ is an operator of trace class and $B$ is a bounded linear operator, there holds ${\rm Tr}(AB)\le {\rm Tr}(A)\|B\|$. If the capacity condition (\ref{capacity condition}) holds with $0<\beta <1,$ then the third term in Proposition \ref{prop: error decomposition in HK} can be bounded as
\begin{equation}\label{e1}
\begin{split}
&4\eta^2(W_+'(0))^2\sum_{t=1}^T\left(\kappa^2\bE_{Z^{t-1}}[\|f_t\|_K^2]+M^2\right){\rm Tr}\left(L_K(I-\eta W_+'(0) L_K)^{2(T-t)}\right)\\
&\leq 2\eta^2(W_+'(0))^2\sum_{t=1}^T\left(2\kappa^2\bE_{Z^{t-1}}[\|f_t\|_K^2]+2M^2\right){\rm Tr}\left(L_K(I-\eta W_+'(0) L_K)^{2(T-t)}\right)\\
&\leq \eta^2\sum_{t=1}^T\left\|L_K^{1-\beta}(I-\eta W_+'(0) L_K)^{2(T-t)}\right\| \\
&\quad \quad \quad\times 2(W_+'(0))^2 {\rm Tr}(L_K^\beta)\left(2\kappa^2\max_{1\le t\le T}\bE_{Z^{t-1}}[\|f_t\|_K^2]+2M^2\right).
\end{split}
\end{equation}
Now we turn to estimate $\eta^2\sum_{t=1}^T\left\|L_K^{1-\beta}(I-\eta W_+'(0) L_K)^{2(T-t)}\right\|$ appeared in the bound above.  One can get from Lemma \ref{lem: key lemma} with $\alpha=1-\beta$ and $s=2t$ that
\begin{align*}
&\eta^2 \sum_{t=1}^T\left\|L_K^{1-\beta}(I-\eta W_+'(0) L_K)^{2(T-t)}\right\|\\
&=\eta^2\|L_K^{1-\beta}\|+\eta^2\sum_{t=1}^{T-1}\left\|L_K^{1-\beta}(I-\eta W_+'(0) L_K)^{2t}\right\|\\
&\le \eta^2\kappa^{2-2\beta}+\eta^2 \left(\frac{1-\beta}{eW_+'(0)}\right)^{1-\beta}\sum_{t=1}^{T-1} \frac{1}{(2\eta t)^{1-\beta}}\\
&\le \eta^2\kappa^{2-2\beta}+\left(\frac{1-\beta}{2eW_+'(0)}\right)^{1-\beta} \frac{1}{\beta} \eta^{1+\beta} T^{\beta}\\
&\le \left( \kappa^{2-2\beta}+\left(\frac{1-\beta}{eW_+'(0)}\right)^{1-\beta}\frac{1}{\beta}\right)\eta^{1+\beta}T^{\beta},
\end{align*}
Combining the bound \eqref{e1} with the bound (\ref{expectation bound for ft}) for $\bE_{Z^{t-1}}\left[\|f_t\|_K^2\right]$, the third term can be bounded as
\begin{equation}\label{the second term in HK}
\begin{split}
&4\eta^2(W_+'(0))^2\sum_{t=1}^T\left(\kappa^2\bE_{Z^{t-1}}[\|f_t\|_K^2]+M^2\right){\rm Tr}\left(L_K(I-\eta W_+'(0) L_K)^{2(T-t)}\right)\\
&\le \left( \kappa^{2-2\beta}+\left(\frac{1-\beta}{eW_+'(0)}\right)^{1-\beta}\frac{1}{\beta}\right) \eta^{1+\beta}T^{\beta}\\
&\quad \quad \quad \times 2(W_+'(0))^2 {\rm Tr}(L_K^\beta)\left(\kappa^2C_4(1+(\eta T)^{2p+3}{\sigma^{-4p}})+M^2\right).
\end{split}
\end{equation}

For the second term $2\eta^2\bE\left[\left\|\sum_{t=1}^T(I-\eta W_+'(0)L_K)^{T-t}E_{\sigma,t}\right\|_K^2\right]$, by the uniform bound (\ref{bound for Esigma}) for $\left\|E_{\sigma,t}\right\|_K$, we have

\begin{equation}\label{the third term in HK}
\begin{split}
&2\eta^2\left\|\sum_{t=1}^T(I-\eta W_+'(0)L_K)^{T-t}E_{\sigma,t}\right\|_K^2\\
&\le 2\eta^2\left(\sum_{t=1}^T\left\|(I-\eta W_+'(0)L_K)^{T-t}E_{\sigma,t}\right\|_K\right)^2 \\
&\le 2\eta^2\left(\sum_{t=1}^T \left\|E_{\sigma,t}\right\|_K\right)^2\le 2\kappa^2 c_p^2\left(M+ \kappa M\sqrt{C_W }\right)^{4p+2} \frac{(\eta T)^{2p+3}}{\sigma^{4p}}.
\end{split}
\end{equation}

Now putting the above estimates (\ref{the first term in HK}), (\ref{the second term in HK}) and (\ref{the third term in HK}) back into Proposition \ref{prop: error decomposition in HK} yields the bound for $\bE_{Z^T}\left[\|f_{T+1}-f_\rho\|_K^2\right]$, which is given by
\begin{align*}
&2\left( \frac{r-\frac12}{eW_+'(0)}\right)^{2(r-\frac12)}\left\| g_\rho\right\|_\rho^2(\eta T)^{2(r-\frac12)})+2\kappa^2 c_p^2\left(M+ \kappa M\sqrt{C_W }\right)^{4p+2} \frac{(\eta T)^{2p+3}}{\sigma^{4p}}\\
&+2(W_+'(0))^2 {\rm Tr}(L_K^\beta)\left(\kappa^2C_4(1+(\eta T)^{2p+3}{\sigma^{-4p}})+M^2\right) \left( \kappa^{2-2\beta}+\left(\frac{1-\beta}{eW_+'(0)}\right)^{1-\beta}\frac{1}{\beta}\right)\eta^{1+\beta}T^{\beta}.
\end{align*}
Finally we choose $\eta=\frac{1}{\eta_0}T^{\frac{1-2r-\beta}{2r+\beta}}$ in the bound above to obtain the desired result with
\begin{align*}
\tilde{C}&=2\left( \frac{r-\frac12}{eW_+'(0)}\right)^{2r-1}\left\| g_\rho\right\|_\rho^2\eta_0^{2r-1}+2\kappa^2 c_p^2\left(M+ \kappa M\sqrt{C_W }\right)^{4p+2} \eta_0^{-(2p+3)}\\
&+2(W_+'(0))^2 {\rm Tr}(L_K^\beta)\left(\kappa^2C_4(1+\eta_0^{-(2p+3)})+M^2\right) \left( \kappa^{2-2\beta}+\left(\frac{1-\beta}{eW_+'(0)}\right)^{1-\beta}\frac{1}{\beta}\right) \eta_0^{-(1+\beta)}.
\end{align*}
The proof is finished. \qed

\section*{Acknowledgments}
The work of Zheng-Chu Guo is supported by Zhejiang Provincial Natural Science Foundation of China [Project No. LR20A010001], National Natural Science Foundation of China [Project Nos. U21A20426 and 12271473], and Fundamental Research Funds for the Central Universities [Project No. 2021XZZX001]. The work of Andreas Christmann is partially supported by German Science Foundation (DFG) under Grant CH 291/3-1. The work of Lei Shi is supported by the National Natural Science Foundation of China [Project Nos.12171039 and 12061160462] and Shanghai Science and Technology Program [Project Nos. 21JC1400600 and 20JC1412700].

\end{document}